\documentclass[12pt]{article}
\usepackage[nonatbib,preprint]{neurips_2019}
\usepackage[sectionbib]{natbib}
\usepackage{chapterbib}

\usepackage{amsmath,amssymb,amsthm,amsfonts}
\usepackage[colorlinks,
            linkcolor=red,
            anchorcolor=red,
            citecolor=blue
            ]{hyperref}
\usepackage[utf8]{inputenc} 
\usepackage[T1]{fontenc}
\usepackage{microtype}
\usepackage{color}
\usepackage{graphicx}
\usepackage{cancel}
\usepackage{bm}
\usepackage{nicefrac}

\setcitestyle{numbers,sort&compress,square,comma}

\newtheorem{theorem}{Theorem}

\newtheorem{corollary}[theorem]{Corollary}

\newtheorem{definition}{Definition}

\usepackage{enumerate}

\theoremstyle{definition}

\usepackage[toc,page]{appendix} 

\theoremstyle{remark}

\newcommand{\R}{\mathbb{R}}
\newcommand{\E}{\mathbb{E}}
\renewcommand{\P}{\Pr}
\newcommand{\dee}{\mathrm{d}}

\newcommand{\Nei}{\mathcal{N}}
\renewcommand{\deg}{d}
\newcommand{\sign}{\mathrm{sgn}}

\renewcommand{\vec}{\ensuremath\bm}
\renewcommand{\vec}[1]{\ensuremath{\mathbf{#1}}}
\newcommand{\mat}[1]{\ensuremath{\mathbf{#1}}}

\title{Additive function approximation in the brain}
\author{Kameron Decker Harris\\
  Paul G.\ Allen School of Computer Science and Engineering, Department of Biology\\
  University of Washington\\
  \texttt{kamdh@uw.edu}
}
\date{\today}

\begin{document}

\maketitle

\begin{abstract}
  Many biological learning systems 
  such as the mushroom body, hippocampus, and cerebellum
  are built from sparsely connected networks of neurons.
  For a new understanding of such networks, 
  we study the function spaces induced by sparse random features and
  characterize what functions may and may not be learned.
  A network with $d$ inputs per neuron is found to be equivalent
  to an additive model of order $d$,
  whereas with a degree distribution the network combines
  additive terms of different orders.
  We identify three specific advantages of sparsity:
  additive function approximation is a powerful 
  inductive bias that limits the curse of dimensionality,
  sparse networks are stable to outlier noise in the inputs, and
  sparse random features are scalable.
  Thus, even simple brain architectures
  can be powerful function approximators.
  Finally, we hope that this work helps
  popularize kernel theories of networks among computational neuroscientists.
\end{abstract}

\begingroup
\let\clearpage\relax
\section{Introduction}

Kernel function spaces are popular among machine learning researchers
as a potentially tractable framework for 
understanding artificial neural networks trained via gradient descent 
\citep[e.g.][]{bach2017,jacot2018,chizat2018,mei2018,rotskoff2018,venturi2018}.
Artificial neural networks 
are an area of intense interest due to their
often surprising empirical performance on a number of challenging problems
and our still incomplete theoretical understanding.
Yet computational neuroscientists have not 
widely applied these new theoretical tools
to describe the ability of biological networks
to perform function approximation.

The idea of using fixed random weights in a neural network is primordial,
and was a part of Rosenblatt's perceptron model of the retina \citep{rosenblatt1958}.
Random features have then resurfaced under many guises:
random centers in radial basis function networks \citep{broomhead1988},
functional link networks \citep{igelnik1995},
Gaussian processes (GPs) \citep{neal1996,williams1997},
and so-called extreme learning machines \citep{wang2008};
see \citep[][]{scardapane2017} for a review.
Random feature networks, 
where the neurons are initialized with random weights and only the 
readout layer is trained, 
were proposed by Rahimi and Recht in order to improve the performance of kernel methods 
\citep{rahimi2008,rahimi2008a}
and can perform well for many problems \citep{scardapane2017}.

In parallel to these developments in machine learning, 
computational neuroscientists have also studied the properties of random networks
with a goal towards understanding neurons in real brains.
To a first approximation, many neuronal circuits seem to be randomly organized
\citep{ganguli2012,caron2013,caron2013a,harris2017,litwin-kumar2017}.
However, the recent theory of random features appears to be mostly unknown to the greater 
computational neuroscience community.

Here, we study random feature networks with {\em sparse connectivity}:
the hidden neurons each receive input from a random, sparse subset of input neurons.
This is inspired by the observation that the connectivity in a variety of 
predominantly feedforward brain networks is
approximately random and sparse.
These brain areas include the cerebellar cortex, invertebrate mushroom body, and 
dentate gyrus of the hippocampus \citep{cayco-gajic2019}.
All of these areas perform pattern separation and associative learning.
The cerebellum is important for motor control, 
while the mushroom body and dentate gyrus 
are general learning and memory areas for invertebrates and vertebrates, respectively,
and may have evolved from a similar structure in the ancient bilaterian ancestor
\citep{wolff2016}.
Recent work has argued that the sparsity observed in these areas may be optimized to
balance the dimensionality of representation 
with wiring cost \citep{litwin-kumar2017}.
Sparse connectivity has been used to compress neural networks 
and speed up computation \citep{han2015,han2015a,wen2016},
whereas convolutions are a kind of structured sparsity
\citep{mairal2014,jones2019}.

We show that sparse random features approximate additive kernels
\citep{wahba1990,bach2009,duvenaud2011,kandasamy2016}
with arbitrary orders of interaction.
The in-degree of the hidden neurons $d$ sets the order of interaction.
When the degrees of the neurons are drawn from a distribution,
the resulting kernel contains a weighted mixture of interactions.
These sparse features offer advantages of
generalization in high-dimensions, stability under perturbations of their input,
and computational and biological efficiency.

\begin{figure}[t!]
  \centering
  \includegraphics[width=\linewidth]{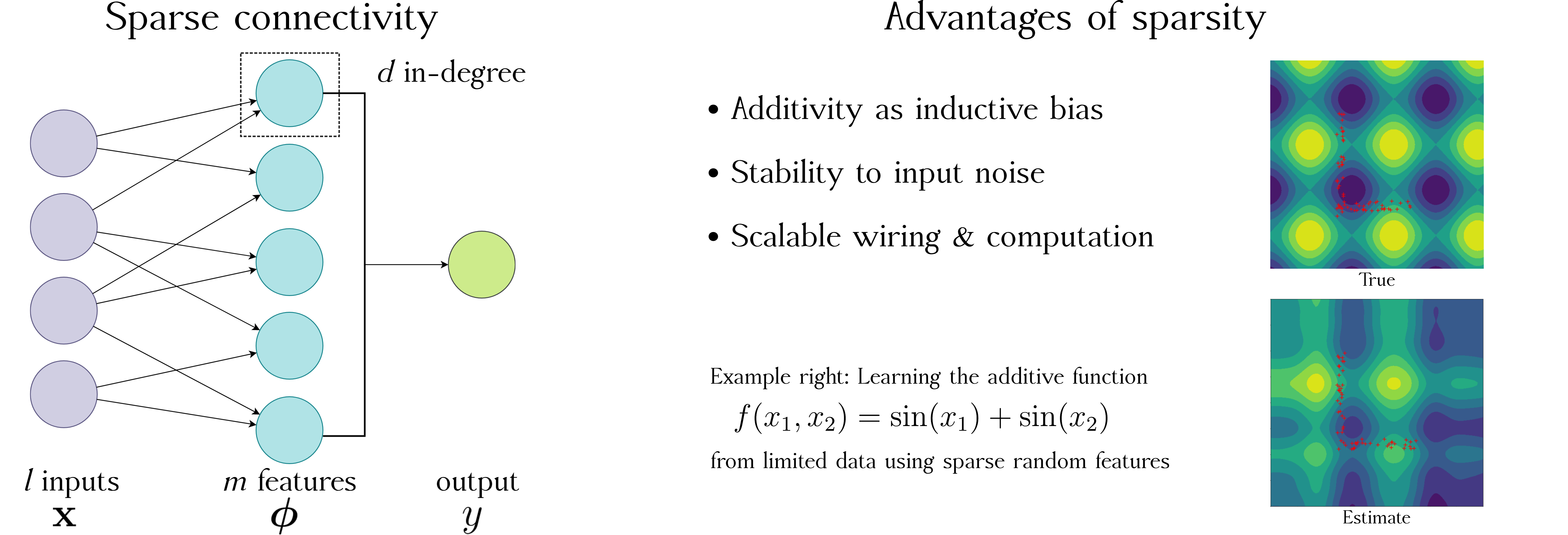}
  \vspace{-1em}
  \caption{Sparse connectivity in a shallow neural network. 
    The function shown is the sparse random feature approximation to an additive sum of sines,
    learned from poorly distributed samples (red crosses).
    Additivity offers structure which may be leveraged for fast and efficient learning.}
  \label{fig:summary}
  \vspace{-1em}
\end{figure}

\section{Background: Random features and kernels}

Now we will introduce the mathematical setting and review how
random features give rise to kernels.
The simplest artificial neural network contains a single hidden layer,
of size $m$, receiving input from a layer of size $l$ (Figure~\ref{fig:summary}).
The activity in the hidden layer is given by, for $i \in [m]$,
\begin{equation}
  \label{eq:features}
  \phi_i(\vec{x}) = h( \vec{w}^\intercal_i \vec{x} + b_i) .
\end{equation}
Here each $\phi_i$ is a feature in the hidden layer, 
$h$ is the nonlinearity,
$\mat{W}=(\vec{w}_1, \vec{w}_2, \ldots, \vec{w}_m) \in \R^{l \times m}$ 
are the input to mixed weights,
and $\vec{b} \in \R^m$ are their biases.
We can write this in vector notation as 
$\bm{\phi}(\vec{x}) = h( \mat{W}^\intercal \vec{x} - \vec{b})$, where
$\bm{\phi}: \R^l \to \R^m$.

Random features networks are draw their
input-hidden layer weights at random.
Let the weights $\mat{W}$ and biases $\vec{b}$ in the feature expansion
\eqref{eq:features}
be sampled i.i.d.\ from a distribution $\mu$ on $\R^{l+1}$.
Under mild assumptions, the inner product of the feature vectors for two inputs
converges to its expectation
\begin{align}
\frac{1}{m} \bm{\phi}(\vec{x})^\intercal \bm{\phi}(\vec{x}') 
\xrightarrow{ m \to \infty} 
\E \left[ \phi(\vec{x}) \phi(\vec{x}') \right] 
=
\int 
h(\vec{w}^\intercal \vec{x} + b) h(\vec{w}^\intercal \vec{x}' + b)\, \dee \mu(\vec{w},b) 
:= k(\vec{x}, \vec{x}').   \label{eq:rf_kernel}
\end{align}
We identify the limit \eqref{eq:rf_kernel} with a reproducing kernel
$k(\vec x, \vec x')$
induced by the random features, since
the limiting function is an inner product and thus always positive semidefinite
\citep{rahimi2008}.
The kernel defines an associated reproducing kernel Hilbert space (RKHS) of functions.
For a finite network of width $m$,
the inner product
$\frac{1}{m} \bm\phi(\vec{x})^\intercal \bm\phi(\vec{x}')$
is a randomized approximation to the kernel $k(\vec x, \vec x')$.

\section{Sparsely connected random feature kernels}

We now turn to our main result:
the general form of the random feature kernels 
with sparse, independent weights.
For simplicity, we start with a regular model and then generalize
the result to networks with varying in-degree.
Two kernels that can be computed in closed form are highlighted.


Fix an in-degree $\deg$, where $1 \leq \deg \leq l$,
and let $\mu | d$ be a distribution on $\R^d$
which induce, 
together with some nonlinearity $h$,
the kernel 
$k_\deg(\vec z, \vec z')$ on $\vec z, \vec z' \in \R^d$
(for the moment, $d$ is not random).
Sample a sparse feature $i \in [m]$ in two steps:
First,
pick $\deg$ neighbors from all $\binom{l}{\deg}$ uniformly at random.
Let $\Nei_i \subseteq [l]$ denote this set of neighbors.
Second, 
sample $w_{ji} \sim \mu | d$ if $j \in \Nei_i$
and otherwise set $w_{ji} = 0$.
We find that the resulting kernel
\begin{equation}
  \label{eq:K_reg}
  k_\deg^\mathrm{reg} ( \vec{x}, \vec{x}' ) 
  = \E [ \E [ \phi( \vec{x}_\Nei ) \phi( \vec{x}'_\Nei ) | \Nei ] ]
  = 
  { \binom{l}{\deg}^{-1} } \sum_{\Nei: |\Nei| = \deg} k_d( \vec{x}_\Nei, \vec{x}'_\Nei ) .
\end{equation}
Here $\vec{x}_\Nei$ denotes the length $\deg$ vector of $\vec{x}$ restricted
to the neighborhood $\Nei$, with the other $l - \deg$ entries in $\vec{x}$ ignored.


More generally, the in-degrees may be
chosen independently according to a {\it degree distribution},
so that $\deg$ becomes a random variable.
Let $D(d)$ be the probability mass function of the hidden node in-degrees.
Conditional on node $i$ having degree $d_i$,
the in-neighborhood $\Nei_i$ is chosen uniformly 
at random among the $\binom{l}{d_i}$ possible sets.
Then the induced kernel becomes
\begin{equation}
  \label{eq:K_deg_dist}
  k_D^\mathrm{dist} (\vec{x}, \vec{x}') 
  =  \E [ \E [ \phi( \vec{x}_\Nei ) \phi( \vec{x}'_\Nei ) | \Nei, d ] ]
  = \sum_{\deg=0}^l D(d) \, k_\deg^\mathrm{reg}( \vec{x}, \vec{x}') .
\end{equation}
For example, if every layer-two node chooses its inputs independently with probability $p$,
the $D(\deg_i)$ is the probability mass function of the binomial distribution
$\mathrm{Bin}(l,p)$. 
The regular model \eqref{eq:K_reg} is a special case of \eqref{eq:K_deg_dist}
with $D(d') = \mathbb{I}{\{ d' = d \}}$.
Extending the proof techniques in \citep{rahimi2008} yields:

{\bf Claim }
{\it 
The random map $\frac{1}{m} \bm \phi(\vec x)^\intercal \bm \phi(\vec x')$ 
with $\kappa$-Lipschitz nonlinearity
uniformly approximates 
$k_D^\mathrm{dist}(\vec x, \vec x')$ 
to error $\epsilon$ 
using 
$m= \Omega ( \frac{l \kappa^2}{\epsilon^2} \log \frac{C}{\epsilon} )$ 
many features
(the proof is contained in Appendix~\ref{app:proof}).
}

\paragraph{Two simple examples}

With Gaussian weights and regular $d=1$, 
we find that (see Appendix~\ref{app:kernels})
\begin{align}
 k_1^\mathrm{reg} (\vec x, \vec x') &= 1 - \frac{1}{l} \| \sign(\vec x) - \sign(\vec x') \|_0 
  & &\mbox{if $h = $ step function, and} 
      \label{eq:sparse_step} \\
 k_1^\mathrm{reg} (\vec x, \vec x') &= 1 - \frac{c}{l}\| \vec x - \vec x' \|_1 & &\mbox{if $h = $ sign function}. 
                                                                                   \label{eq:sparse_sign}
\end{align}

\section{Advantages of sparse connectivity}


\subsection{Additive modeling}

The regular degree kernel 
\eqref{eq:K_reg}
is a sum of kernels that only depend on combinations of $d$ inputs,
making it an {\it additive kernel} of order $d$.
The general expression for the degree distribution kernel
\eqref{eq:K_deg_dist}
illustrates that sparsity leads to a mixture of additive kernels 
of different orders.
These have been referred to as additive GPs
\citep{duvenaud2011}, but these kind of 
models have a long history as
generalized additive models
\citep[e.g.][]{wahba1990,hastie2009}.
For the regular degree model with $d=1$,
the sum in \eqref{eq:K_reg} is over neighborhoods of size one,
simply the individual indices of the input space.
Thus, for any two input neighborhoods $\Nei_1$ and $\Nei_2$, we have
$|\Nei_1 \cap \Nei_2| = \emptyset$,
and
the RKHS corresponding to 
$k_{1}^\mathrm{reg} (\vec{x},\vec{x}')$
is the direct sum of the subspaces
$\mathcal H = \mathcal H_1 \oplus \ldots \oplus \mathcal H_l$.
Thus regular $d=1$ defines a first-order additive model,
where $f(\vec x) = f_1(x_1) + \ldots + f_l (x_l)$.
When $d>1$ we allow
interactions between subsets of $d$ variables,
e.g.\ regular $d=2$ leads to
$f(\vec x) = f_{12}(x_1, x_2) + \ldots 
+ f_{l-1, l}(x_{l-1}, x_l)$,
all pairwise terms.
These interactions are defined by the structure of the terms
$k_d(\vec{x}_\Nei, \vec{x}'_\Nei)$.
Finally, the degree distribution $D(d)$ determines how much
weight to place on different degrees of interaction.

\paragraph{Generalization from fewer examples in high dimensions}

Stone proved that first-order additive models
do not suffer from the curse of dimensionality
\citep{stone1985,stone1986}, as 
the excess risk does not depend on the dimension $l$.
\citet{kandasamy2016} extended this result 
to $d$th-order additive models and found
a bound on the excess risk of 
$O (l^{2d} n^{\frac{-2s}{2s + d}} )$
or 
$O (l^{2d} C^d / n )$ for kernels
with polynomial or exponential eigenvalue decay rates
($n$ is the number of samples and 
the constants $s$ and $C$ parametrize rates).
Without additivity, these weaken to
$O (n^{\frac{-2s}{2s + l}} )$
and
$O (C^l / n )$,
much worse when $l \gg d$.

\paragraph{Similarity to dropout}

Dropout regularization \citep{hinton2012,srivastava2013}
in deep networks
has been analyzed in a kernel/GP framework
\citep{duvenaud2014},
leading to \eqref{eq:K_deg_dist}
with $D = \mathrm{Bin}(l,p)$
for a particular base kernel.
Dropout may thus improve generalization by enforcing approximate additivity,
for the reasons above.

\subsection{Stability: robustness to noise or attacks affecting a few inputs}

Equations \eqref{eq:sparse_step} and \eqref{eq:sparse_sign}
are similar:
They differ only by the presence of an $\ell^0$-``norm''
versus an $\ell^1$-norm and the presence of the sign function.
Both norms are 
stable to outlying coordinates in an input $\vec x$.
This property also holds for $1 < d \ll l$,
since every feature $\phi_i(\vec x)$ only depends on $d$
inputs, and therefore only a minority of the $m$ 
features will be affected by the few outliers.\footnote{
If one coordinate of $\vec x$ is noisy,
the probability that the $i$th neuron is affected is
$d_i/l \ll 1$.}
Thus, sufficiently sparse features will be less
affected by sparse noise than a fully-connected network.
Furthermore, any regressor
$f(\vec x) = \bm\alpha^\intercal \bm\phi(\vec x)$
built from these features will also be stable.
By Cauchy-Schwartz,
$ | f(\vec x) - f(\vec x') | 
  \leq 
  \| \bm\alpha \|_2
  \| \bm\phi(\vec x) - \bm\phi(\vec x') \|_2 
$.
Thus if $\vec x' = \vec x + \vec e$ where $\vec e$ 
is noise with small number of nonzero entries, then
$f(\vec x') \approx f(\vec x)$
since 
$\bm\phi(\vec x) \approx \bm\phi(\vec x')$.
The network intrinsically denoises its inputs,
which may offer advantages \citep[e.g.][]{litwin-kumar2017}.
Stability also may guarantee
the robustness of networks to adversarial attacks
\citep{lecuyer2018,cohen2019,salman2019},
thus sparse networks are robust to attacks on only a few inputs.


\subsection{Scalability: computational and biological}

\label{sec:scaling}

\paragraph{Computational}
Sparse random features give potentially huge improvements in scaling.
Direct implementations of additive models incur a large cost for $d > 1$,
since \eqref{eq:K_reg} requires a sum over
$\binom{l}{d} = O(l^d)$ neighborhoods.\footnote{
There is a more efficient method when working with a {\em tensor product kernel},
as in
\citep{bach2009,duvenaud2011,kandasamy2016}.}
This leads to $O(n^2 l^d)$ time to compute the Gram matrix of $n$ examples
and $O(n l^d)$ operations to evaluate $f(\vec x)$.
In our case,
since the random features method is primal,
we need to perform
$O(n m d)$ computations to evaluate the feature matrix
and the cost of evaluating $f(\vec x)$
remains $O(m d)$.\footnote{
Note that we need to take $m = \Omega(l)$ to ensure
good approximation of the kernel (Appendix~\ref{app:proof}).}
Sparse matrix-vector multiplication makes evaluation faster than  the
$O(ml)$ time it takes when connectivity is dense.
For ridge regression, we have the usual advantages that 
computing an estimator takes
$O(n m^2 + n m d)$ time and $O(n m + m d)$ memory,
rather than $O(n^3)$ time and $O(n^2)$ memory for 
a na\"ive kernel ridge method.

\paragraph{Biological}
In a small animal such as a flying insect, space is extremely limited.
Sparsity offers a huge advantage in terms of wiring cost \citep{litwin-kumar2017}.
Additive approximation also means
that such animals can learn much more quickly,
as seen in the mushroom body \citep{huerta2009,delahunt2019,delahunt2018a}.
While the previous computational points do not apply as well to biology,
since real neurons operate in parallel,
fewer operations translate into lower metabolic cost for the animal.

\section{Discussion}

Inspired by their ubiquity in biology,
we have studied sparse random networks of neurons using
the theory of random features,
finding the advantages of 
additivity, stability, and scalability.
This theory shows that sparse networks such as those found in the mushroom body,
cerebellum, and hippocampus can be powerful function approximators.
Kernel theories of neural circuits may be more broadly applicable 
in the field of computational neuroscience.

\paragraph{Expanding the theory of dimensionality in neuroscience}
Learning is easier in
additive function spaces because they are {\em low-dimensional},
a possible explanation for few-shot learning in biological systems.
Our theory is complementary to existing theories of dimensionality
in neural systems 
\citep[][]{ganguli2012,rigotti2013,babadi2014,meister2015,mazzucato2016,litwin-kumar2017,gao2017,mastrogiuseppe2018,farrell2019},
which defined dimensionality using a (debatably) ad hoc 
skewness measure of covariance eigenvalues.
Kernel theory extends this concept,
measuring dimensionality similarly \citep{zhang2005} 
in the space of nonlinear functions spanned by the kernel.

\paragraph{Limitations}

We model biological neurons as simple scalar functions, 
completely ignoring time and neuromodulatory context.
It seems possible
that a kernel theory could be developed for time- and context-dependent features.
Our networks suppose i.i.d.\ weights,
but weights that follow Dale's law should also be considered.
We have not studied the sparsity of activity, postulated to be relevant in cerebellum.
It remains to be demonstrated how the theory can make concrete, testable predictions,
e.g.\ whether this theory may explain identity versus concentration encoding of odors
or the discrimination/generalization tradeoff under experimental conditions.


{\small
\paragraph{Acknowledgments}
KDH was supported by a Washington Research Foundation postdoctoral fellowship.
Thank you to Rajesh Rao for support during this project and to
Bing Brunton for support and many helpful comments.
}

{
\setlength{\bibsep}{0pt plus 0.3ex}
\small
\bibliographystyle{unsrtnat}
\bibliography{Library}
}

\endgroup
\newpage
\appendix
{\Large \bf Appendices: Additive function approximation in the brain}

{\bf Table of contents}
\begin{itemize}
\item Appendix~\ref{app:experiments}: Test problems and numerical experiments
\item Appendix~\ref{app:kernels}: Kernel examples arising from random features, dense and sparse
\item Appendix~\ref{app:proof}: Kernel approximation results, uniform convergence of Lipschitz features
\end{itemize}

\section{Test problems}
\label{app:experiments}

We have implemented sparse random features in Python to demonstrate
the properties of learning in this basis.
Our code, 
which provides a {\tt scikit-learn} style
{\tt SparseRFRegressor} and
{\tt SparseRFClassifier} 
estimators,
is available from
\url{https://github.com/kharris/sparse-random-features}.

\subsection{Additive function approximation}

\subsubsection{Comparison with datasets from \citet{kandasamy2016}}
As said in the main text,
\citet{kandasamy2016}
created a theory of the generalization properties of
higher-order additive models.
They supplemented this with an empirical study of a number of datasets
using their Shrunk Additive Least Squares Approximation (SALSA)
implementation of the additive kernel ridge regression (KRR).
Their data and code were obtained from
\url{https://github.com/kirthevasank/salsa}.

We compared the performance of SALSA 
to the sparse random feature approximation
of the same kernel.
We employ random sparse Fourier features
with Gaussian weights $N(0, \sigma^2 \mat I)$
with $\sigma = 0.05 \cdot \sqrt{d} n^{1/5}$
in order to match the Gaussian radial basis function 
used by \citet{kandasamy2016}.
We use $m = 300 l$ features for every problem,
with regular degree $d$ selected equal to the one chosen by SALSA.
The regressor on the features is cross-validated ridge regression 
({\tt RidgeCV} from {\tt scikit-learn})
with ridge penalty selected from 5 logarithmically spaced points
between $10^{-4} \cdot n$ and $10^2 \cdot n$.

\begin{figure}[b!]
  \centering
  \includegraphics[width=\linewidth,trim={20 20 20 20},clip]
  {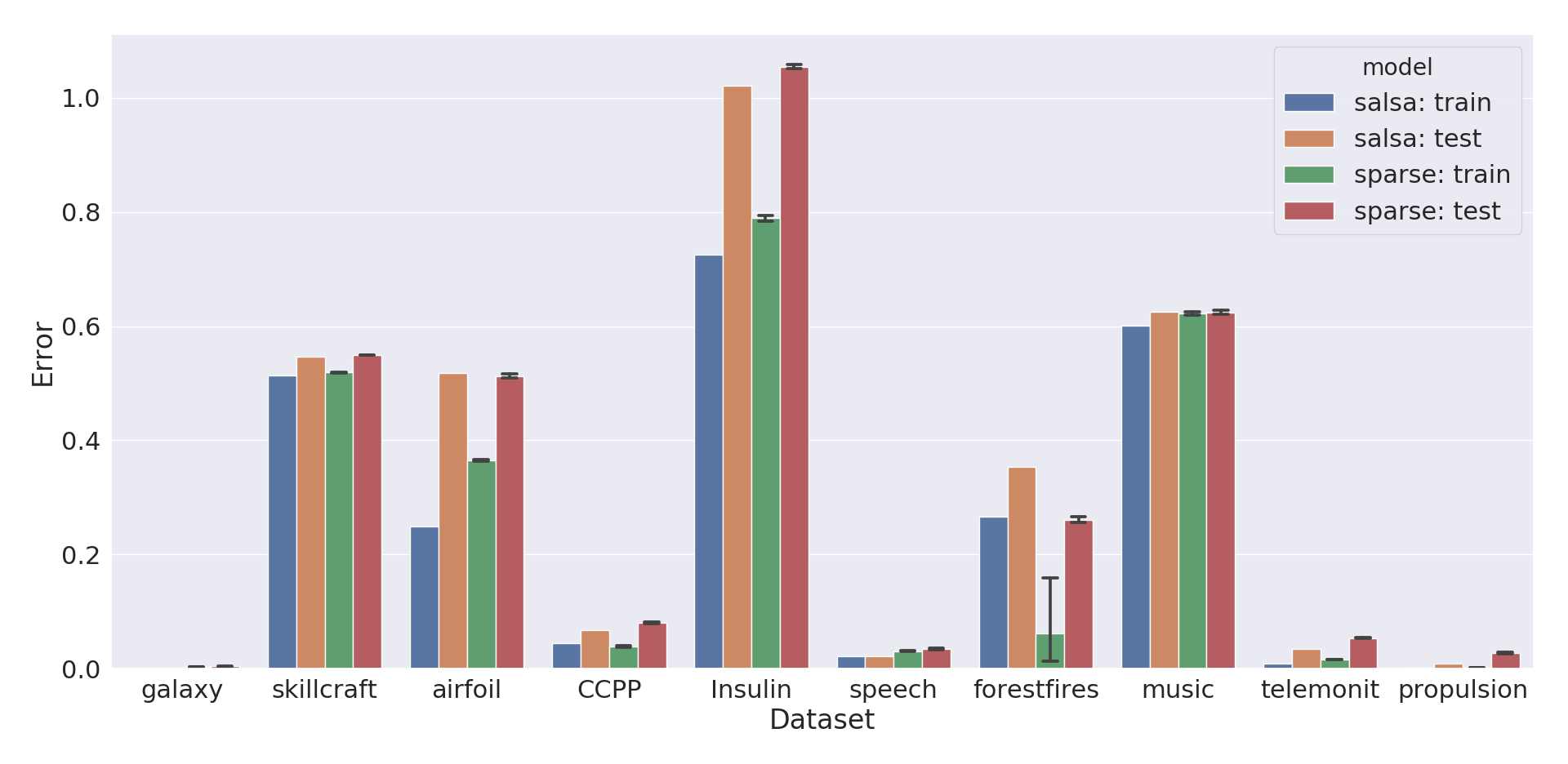}
  \caption{Comparison of sparse random feature approximation to additive kernel method SALSA
  \citep{kandasamy2016}. The parameters were matched between the two models (see text).
  The sparse feature approximation performs slightly worse than the exact method,
  but similar.}
  \label{fig:salsa}
\end{figure}

In Figure~\ref{fig:salsa}, we compare the performance of sparse
random features to SALSA.
Generally, the training and testing errors of the sparse model
are slightly higher than for the kernel method,
except for the {\tt forestfires} dataset.

\subsubsection{Polynomial test function shows generalization from fewer examples}

We studied the speed of learning for a test function as well.
The function to be learned $f(\vec x)$
was a sparse polynomial plus a linear term:
\[
f(\vec x) = 
c_1 \vec a^\intercal \vec x + 
c_2 \, p(\vec x) .
\]
The linear term took 
$\vec a \sim N(0, \mat I)$,
the polynomial $p$
was chosen to have 3 terms of degree 3
with weights drawn from $N(0,1)$.
The inputs $\vec x$ are drawn from 
the uniform distribution over 
$[0,1]^{16}$.
Gaussian noise $\epsilon$ with variance $0.05^2$ was added to generate
observations $y_i = f(\vec x_i) + \epsilon_i$.
Constants $c_1$ and $c_2$ were tuned
by setting 
$c_1 = \frac{1}{\sigma_\mathrm{lin}} \frac{1-\alpha}{\sqrt{\alpha^2 + (1-\alpha)^2}} $
and 
$c_2 = \frac{1}{ \sigma_\mathrm{nonlin}} \frac{\alpha}{\sqrt{\alpha^2 + (1-\alpha)^2}}$,
where
$\alpha = 0.05$
and
$\sigma_\mathrm{lin}$ and $\sigma_\mathrm{nonlin}$ were the 
standard deviations of the linear and nonlinear terms alone.

For this problem we use random features of varying 
regular degrees $d = 1,3,10,16$ and number of data points $n$.
The features use a
Fourier nonlinearity $h(\cdot) = (\sin \cdot, \cos \cdot)$,
weights $w_{ij} \sim N(0, d^{-1/2})$,
and biases $b_i \sim U([-\pi, \pi])$,
leading to an RBF kernel in $d$ dimensions.
The output regression model is again ridge regression
with the penalty selected via cross-validation on the training set
from 7 logarithmically spaced points between
$10^{-4}$ and $10^2$.

\begin{figure}[t!]
  \centering
  \includegraphics[width=0.49\linewidth,trim={10 0 10 0},clip]{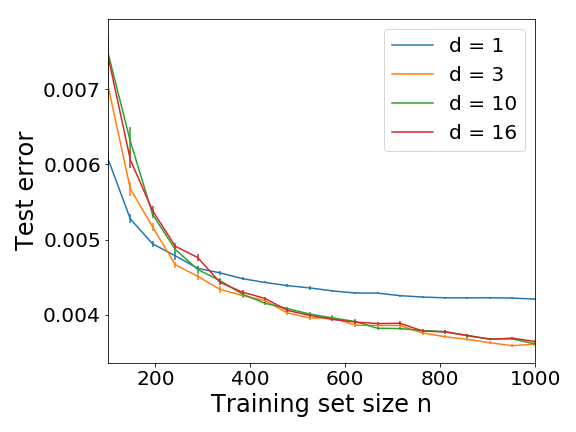}
  \hfill
  \includegraphics[width=0.49\linewidth,trim={10 0 10 0},clip]{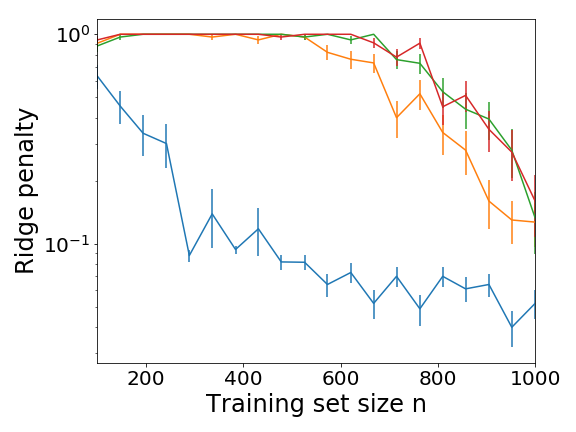}
  \caption{Performance of sparse random features of differing degree $d$
    and training size $n$ for the polynomial test function. 
    Test error is measured as mean square error with noise floor at 0.0025.
    As the amount of training data increases, higher $d$,
    i.e.\ more model complexity, is preferred.
  }
  \label{fig:test_fun}
\end{figure}

In Figure~\ref{fig:test_fun}, we show the test error as well as the selected
ridge penalty for different values of $d$ and $n$.
With a small amount of data ($n < 250$), the model with $d=1$ has the lowest
test error, since this ``simplest'' model is less likely to overfit.
On the other hand, in the intermediate data regime ($250 < n < 400$),
the model with $d = 3$ does best.
For large amounts of data ($n > 400$), 
all of the models with interactions
$d \geq 3$ do roughly the same. 
Note that with the RBF kernel the RKHS 
$\mathcal{H}_d \subseteq \mathcal{H}_{d'}$ 
whenever
$d \leq d'$, so $d > 3$
can still capture the degree 3 polynomial model.
However, we see that the more complex models have a higher ridge penalty selected.
The penalty is able to adaptively control this complexity given enough data.

\subsection{Stability with respect to sparse input noise}

Here we show that sparse random features are stable
for spike-and-slab input noise.
In this example, the truth follows a linear model,
where we have
random input points 
$\vec{x}_i \sim \mathcal{N}(0,\mat{I})$
and linear observations 
$y_i = \vec{x}_i^\intercal \beta$
for 
$i = 1, \ldots, n$
and
$\beta \sim \mathcal{N}(0, \mat{I})$.
However, we only have 
access to sparsely corruputed inputs 
$\vec w_i = \vec x_i + \vec e_i$, 
where 
$\vec e_i = 0$ 
with probability $1-p$ 
and $\vec e_i = \bm \epsilon_x - \vec x_i$ 
with probability $p$, 
$\bm\epsilon_x \sim \mathcal{N}(0, \sigma^2 \mat I)$.
That is, the corrupted inputs are replaced with pure noise.
We use $p = 0.03 \ll 1$ and $\sigma = 6 \gg 1$
so that the noise is sparse but large when it occurs.

\begin{table}[h!]
\centering
\begin{tabular}{l | llll}
Model              & Training score & Testing score       \\
\hline
\hline
Linear   & 0.854          & 0.453           \\
Kernel         & 1.000          & 0.607            \\
Trim + linear      & 0.945          & 0.686           \\
Huber              & 0.858          & 0.392          
\end{tabular}
\caption{Scores ($R^2$ coefficient) 
  of various regression models on linear data with corrupted 
  inputs. In the presence of these errors, linear regression fails to acheive as good a test 
  score as the kernel method, which is almost as good as trimming before performing regression
  and better than the robust Huber estimator.
}
\label{tab:stability}
\end{table}

In Table~\ref{tab:stability} we show the performance of various methods
on this regression problem given the corrupted data $(\mat W, \vec y)$.
Note that if the practitioner has access to the uncorrupted data $\mat X$,
linear regression succeeds with a perfect score of 1.
Using kernel ridge regression
with $k(\vec x, \vec x') = 1 - \frac{1}{l}\| \vec x - \vec x' \|_1 $,
the kernel that arises from sparse random features with $d=1$ and sign nonlinearity,
leads to improved performance over na\"ive linear regression on the corrupted data
or a robust Huber loss function.
The best performance is attained by trimming the outliers and then performing linear regression.
However, this is meant to illustrate our point that sparse random features and their corresponding
kernels may be useful when dealing with noisy inputs in a learning problem.

\begin{figure}[t!]
  \centering
  \includegraphics[width=0.49\linewidth,trim={8 0 10 0},clip]{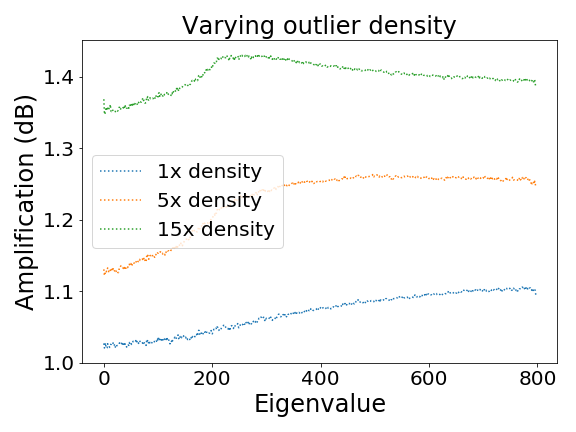}
  \hfill
  \includegraphics[width=0.49\linewidth,trim={8 0 10 0},clip]{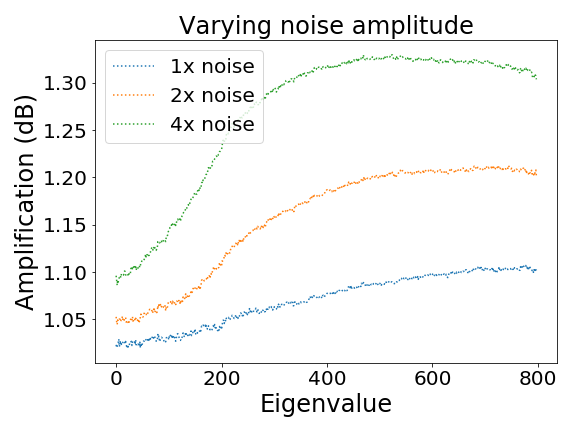}
  \caption{
    Kernel eigenvalue amplification while
    {\bf (left)} varying $p$ with $\sigma = 6$ fixed, and
    {\bf (right)} varying $\sigma$ with $p = 0.03$ fixed.
    Plotted is the ratio of eigenvalues of the kernel matrix
    corrupted by noise to those without any corruption,
    ordered from largest to smallest in magnitude.
    We see that the sparse feature kernel shows little noise amplification
    when it is sparse (right), even for large amplitude.
    On the other hand, less sparse noise does get amplified (left).
  }
  \label{fig:stability}
\end{figure}

In Figure~\ref{fig:stability} we show another way of measuring this stability property.
We compute the eigenvalues of the kernel matrix
on a fixed dataset of size $n = 800$ points
both with noise and without noise.
Plotted are the ratio of the noisy to noiseless eigenvalues, in decibels,
which we call the amplification and is a measure of how corrupted the
kernel matrix is by this noise.
The main trend that we see is, for fixed $p=3$, 
changing the amplitude of the noise $\sigma$ does
not lead to significant amplification, especially of the 
early eigenvalues which are of largest magnitude.
On the other hand, making the outliers denser does lead to more amplification
of all the eigenvalues.
The eigenspace spanned by the largest eigenvalues is the most ``important'' for
any learning problem.

\section{Kernel examples}
\label{app:kernels}

\subsection{Fully-connected weights}

We will now describe a number of common random features and the 
kernels they generate with fully-connected weights.
Later on, we will see how these change as sparsity is introduced 
in the input-hidden connections.

\paragraph{Translation invariant kernels}

The classical random features
\citep{rahimi2008}
sample Gaussian weights
$\vec{w} \sim N(0, \sigma^{-2} \mat{I})$, 
uniform biases
$b \sim U[-a,a]$,
and employ the 
Fourier nonlinearity
$h(\cdot) = \cos(\cdot)$.
This leads to the Gaussian radial basis function kernel
\[
k(\vec{x},\vec{x}') = \exp \left( -\frac{1}{2\sigma^2} \| \vec{x} - \vec{x}' \|^2 \right),
\]
for $\vec{x}, \vec{x}' \in [-a,a]^l$. 
In fact, every translation-invariant kernel arises 
from Fourier nonlinearities for some distributions of weights and biases
(B\^{o}chner's theorem).

\paragraph{Moment generating function kernels}

The exponential function is 
more similar to the kinds of monotone firing rate curves found in 
biological neurons.
In this case, we have
\[
k(\vec{x}, \vec{x}') = \E \exp(\vec{w}^\intercal (\vec{x} + \vec{x}') + 2b) .
\]
We can often evaluate this expectation using moment generating functions.
For example, if $\vec{w}$ and $b$ are independent,
which is a common assumption, then 
\[
k(\vec{x}, \vec{x}') = \E \left( \exp(\vec{w}^\intercal (\vec{x} + \vec{x}') \right) 
\cdot \E \exp(2 b),
\]
where
$\E \left( \exp(\vec{w}^\intercal (\vec{x} + \vec{x}') \right)$
is the moment generating function for the marginal distribution of $\vec{w}$,
and $\E \exp(2 b)$ is just a constant that scales the kernel.

For multivariate Gaussian weights 
$\vec{w} \sim N(\vec{m}, \mat{\Sigma})$ 
this becomes
\begin{equation*}
k(\vec{x}, \vec{x}') = 
\exp \left( \vec{m}^\intercal (\vec{x}+\vec{x}') + 
  \frac{1}{2} (\vec{x} + \vec{x}')^\intercal \mat{\Sigma} (\vec{x} + \vec{x}')  \right)
\cdot \E \exp(2b) .
\end{equation*}
This equation becomes more interpretable if $\vec{m} = 0$ and 
$\mat{\Sigma} = \sigma^{-2} \mat{I}$ 
and the input data are normalized: 
$\| \vec{x} \| = \| \vec{x}' \| = 1$.
Then, 
\[
k(\vec{x}, \vec{x}')
\propto \exp \left( \sigma^{-2} \vec{x}^\intercal \vec{x}' \right)
\propto \exp \left( -\frac{1}{2 \sigma^2} \| \vec{x} - \vec{x}' \|^2 \right).
\]
This result highlights that
dot product kernels 
$k(\vec{x}, \vec{x}')= v(\vec{x}^\intercal \vec{x}')$ ,
where $v: \R \to \R$,
are radial basis functions on the sphere 
$S^{l-1} = \{\vec{x} \in \R^l : \| \vec{x} \|_2 = 1\}$.
The eigenbasis of these kernels are the spherical harmonics \citep{smola2001,bach2017}.

\paragraph{Arc-cosine kernels}

This class of kernels is also induced by monotone ``neuronal'' nonlinearities and leads to 
different radial basis functions on the sphere  \citep{smola2001,cho2009,cho2011}.
Consider standard normal weights 
$\vec{w} \sim N(0, \mat{I})$
and nonlinearities which are threshold polynomial functions
\[
 h(z) = \Theta(z) z^p
\]
for $p \in \mathbb{Z}^+$,
where $\Theta(\cdot)$ is the Heaviside step function.
The kernel in this case is given by
\begin{align*}
k(\vec{x}, \vec{x}') &= 2 \int_{\R^l} 
\Theta(\vec{w}^\intercal \vec{x}) 
\Theta(\vec{w}^\intercal \vec{x}') 
(\vec{w}^\intercal \vec{x})^p 
(\vec{w}^\intercal \vec{x}')^p 
\;
\frac{e^{\frac{-\| \vec{w} \|^2}{2}}}{(2\pi)^{l/2}}
\, \dee \vec{w}
\\
&= \frac{1}{\pi} \| \vec{x} \|^p \| \vec{x}' \|^p J_p(\theta) ,
\end{align*}
for a known function $J_p(\theta)$ where 
$\theta = \mathrm{arccos} \left( \frac{\vec{x}^\intercal \vec{x}' }{\| \vec{x} \| \| \vec{x}' \|} \right)$.
Note that arc-cosine kernels are also dot product kernels.
Also, if the weights are drawn as $\vec w \sim N(0, \sigma^{-2} \vec I)$,
the terms $\vec x$ are replaced by $\vec x / \sigma$, but this does not affect $\theta$.
With $p=0$, corresponding to the step function nonlinearity,
we have $J_0 (\theta) = \pi - \theta$, and the resulting kernel does not depend on 
$\| \vec x \|$ or $\| \vec x' \|$:
\begin{align}
  k(\vec{x}, \vec{x})
  &= 1 - \frac{1}{\pi} \mathrm{arccos} 
    \left( \frac{\vec{x}^\intercal \vec{x}'}{\|\vec{x}\| \|\vec{x}'\|} \right).
   \label{eq:K_step_dense}
\end{align}

\paragraph{Sign nonlinearity}

We also consider a shifted version of the step function nonlinearity,
the sign function $\sign(z)$, equal to $+1$ when $z > 0$,
$-1$ when $z < 0$,
and zero when $z = 0$.
Let $b \sim U([a_1,a_2])$ and $\vec{w} \sim P$,
where $P$ is any spherically symmetric distribution, such as a Gaussian.
Then,
\begin{align}
  k(\vec{x}, \vec{x}') 
  &=  \E \left[ \int_{a_1}^{a_2} \frac{\dee b}{a_2-a_1} \, 
    \sign(\vec{w}^\intercal \vec{x}-b)\, \sign(\vec{w}^\intercal \vec{x}'-b) \right] \nonumber \\
  &= \E \left[ 1 - 2 \, \frac{|\vec{w}^\intercal \vec{x} - \vec{w}^\intercal \vec{x}'|}{a_2 - a_1} \right] 
    \nonumber \\
  &= 1 -  \frac{2}{a_2 - a_1} \E |\vec{w}^\intercal (\vec{x} - \vec{x}')| \nonumber \\
  &= 1 - 2 \E (| \vec{w}^\intercal \vec{e}|) \frac{ \| \vec{x}-\vec{x}'\|_2}{a_2-a_1} \nonumber 
\end{align}
where $\vec{e} = (\vec{x}-\vec{x}')/\|\vec{x}-\vec{x}'\|_2$.
The factor $\E (|\vec{w}^\intercal \vec{e}|)$ 
in front of the norm is just a function of the radial part of the distribution $P$,
which we should set inversely proportional to $\sqrt{l}$
to match the scaling of $\|\vec{x} - \vec{x}'\|_2$.
For $\vec w \sim N(0,\sigma^2 l^{-1} \mathbf{I})$, we obtain
\begin{equation}
  \label{eq:K_sign_dense}
  k(\vec{x}, \vec{x}')
  = 1 -  2 \sigma \sqrt{\frac{2}{\pi l}}  \, \frac{\| \vec{x}-\vec{x}' \|_2}{a_2-a_1} .
\end{equation}

\subsection{Sparse weights}

The sparsest networks possible have $d=1$, 
leading to first-order additive kernels.
Here we look at two simple nonlinearities where we
can perform the sum and obtain an explicit formula 
for the additive kernel.
In both cases, the kernels are simply related to
a robust distance metric.
This suggests that such kernels may be useful in cases 
where there are outlier coordinates in the input data.

\paragraph{Step function nonlinearity}

We again consider the step function nonlinearity $h(\cdot) = \Theta(\cdot)$,
which in the case of fully-connected Gaussian weights leads to the 
degree $p=0$ arc-cosine kernel
$k(\vec{x}, \vec{x}') = 1 - \frac{\theta(\vec x, \vec x')}{\pi}$.
When $d=1$,
$\vec{x}_\Nei = x_i$ and $\vec{x}_\Nei' = x'_i$ are scalars.
For a scalar $a$, normalization leads to $a / \|a\| = \sign(a)$.
Therefore, $\theta = \mathrm{arccos} \left( \sign(x_i) \, \sign(x'_i) \right) = 0$
if $\sign(x_i) = \sign(x'_i)$ and $\pi$ otherwise.
Performing the sum in \eqref{eq:K_reg}, we find that the kernel becomes 
\begin{equation}
k_{1}^\mathrm{reg} ( \vec{x}, \vec{x}') 
= 1 - \frac{\left| \left\{i:  \sign(x_i) \not= \sign(x'_i) \right\} \right| }{l} 
= 1 - \frac{\| \sign(\vec{x}) - \sign(\vec{x}') \|_0}{l} \, .
\label{eq:arccos_1}
\end{equation}
This kernel is equal to one minus the normalized
Hamming distance of vectors $\sign(\vec{x})$ and $\sign(\vec{x}')$.
The fully-connected kernel, on the other hand, 
uses the full angle between the vectors $x$ and $x'$.
The sparsity can be seen as inducing a ``quantization,'' 
via the sign function, on these vectors.
Finally, if the data are in the binary hypercube, 
with $\vec{x}$ and $\vec{x}' \in \{-1, +1\}^l$,
then the kernel is exactly one minus the normalized Hamming distance.


\paragraph{Sign nonlinearity}

We now consider a slightly different nonlinearity, the sign function.
It will turn out that the kernel is quite different than for the step function.
This has $h(\cdot) = \sign(\cdot) = 2\Theta(\cdot) - 1$.
Let $b \sim U([a_1,a_2])$ and $w \sim P$.
Then,
\begin{align}
  k_1^\mathrm{reg} (\vec{x}, \vec{x}') 
  &= \frac{1}{l} \sum_{i=1}^l 
             \E_P \left[ \int_{a_1}^{a_2} \frac{\dee b}{a_2-a_1} \, 
             \sign(w x_i - b) \sign(w x'_i - b) \right] \nonumber \\
  &= \frac{1}{l} \sum_{i=1}^l 
    \E_P \left[ 1 - 2 \, \frac{|wx_i - w x'_i|}{a_2-a_1} \right] \nonumber \\
  &= 1 -  \frac{2 \E_P ( |w| )}{l} \, \frac{\| x - x' \|_1}{a_2 - a_1} . \label{eq:stump}
\end{align}
Choosing $P(w) = \frac{1}{2} \delta(w+1) + \frac{1}{2} \delta(w-1)$
and $a_2 = -a_1 = a$
recovers the  ``random stump'' result of \citet{rahimi2008}.
Despite the fact that sign is just a shifted version of the step function,
the kernels are quite different:
the sign nonlinearity does not exhibit the quantization effect
and depends on the $\ell^1$-norm rather than the $\ell^0$-``norm''.

\section{Kernel approximation results}
\label{app:proof}

We now show a basic uniform convergence result for any random features,
not necessarily sparse,
that use Lipschitz continuous nonlinearities. 
Recall the definition of a Lipschitz function:
\begin{definition}
  A function $f: \mathcal{X} \to \R$ is said to be {\bf $L$-Lipschitz continuous}
  (or Lipschitz with constant $L$)
  if 
  \[
  |f(\vec x) - f(\vec y)| \leq L \|\vec x - \vec y\|
  \]
  holds for all $\vec x, \vec y \in \mathcal{X}$.
  Here, $\|\cdot\|$ is a norm on $\mathcal{X}$
  (the $\ell^2$-norm unless otherwise specified).
\end{definition}

Assuming that $h$ is Lipschitz and 
some regularity assumptions on the distribution $\mu$,
the random feature expansion approximates the kernel uniformly over $\mathcal{X}$.
As far as we are aware, this result has not been stated previously,
although it appears to be known
(see \citet{bach2017a})
and is very similar to Claim 1 in \citet{rahimi2008}
which holds only for random Fourier features
(see also \citet{sutherland2015} and \citet{sriperumbudur2015} for improved results
in this case).
The rates we obtain for Lipschitz nonlinearities are not essentially 
different than those obtained in the Fourier features case.

As for the examples we have given, the only ones which are not Lipschitz
are the step function (order 0 arc-cosine kernel) and sign nonlinearities.
Since these functions are discontinuous, 
their convergence to the kernel occurs in a weaker than uniform sense.
However, our result does apply to the rectified linear nonlinearity
(order 1 arc-cosine kernel), which is non-differentiable at zero 
but 1-Lipschitz and widely applied in artificial neural networks.
The proof of the following Theorem appears at the end of this section.

\begin{theorem}[Kernel approximation for Lipschitz nonlinearities]
  Assume that $\vec{x} \in \mathcal{X} \subset \R^l$
  and that $\mathcal{X}$ is compact, $\Delta = \mathrm{diam}(\mathcal{X})$, and 
  the null vector $0 \in \mathcal{X}$.
  Let the weights and biases
  $(\vec w, b)$ 
  follow the distribution $\mu$
  on $\R^{l +1}$
  with finite second moments.
  Let
  $h(\cdot)$ be a
  nonlinearity which is $L$-Lipschitz continuous
  and define the random feature
  $\phi: \R^l \to \R$
  by 
  $\phi(\vec x) = h( \vec w^\intercal \vec x - b)$.
  We assume that
  the following hold for all $\vec x \in \mathcal{X}$:
  $|\phi(\vec x)| \leq \kappa$ almost surely,
  $\E \,|\phi(\vec x)|^2 < \infty$, and
  $\E \;\phi(\vec x) \phi(\vec x') = k(\vec x, \vec x')$.

  Then
  $
  \sup_{\vec x, \vec x' \in \mathcal{X}}
  \left| 
    \frac{1}{m} \bm{\phi}(\vec x)^\intercal \bm{\phi} (\vec x') 
    -
    k (\vec x, \vec x')
  \right| 
  \leq 
  \epsilon 
  $
  with probability
  at least
  \[
  1 
  - 
  2^8 \left( \frac{\kappa L \Delta \sqrt{\E \|\vec w \|^2 + 3 (\E \|\vec w \|)^2}}{\epsilon} \right)^2 
  \exp \left( \frac{-m \epsilon^2}{4 (2 l + 2) \kappa^2} \right).
  \]
  \label{thm:approx}
\end{theorem}

\paragraph{Sample complexity}

Theorem~\ref{thm:approx} guarantees uniform approximation up to error $\epsilon$ 
using 
$m = 
\Omega\left(\frac{l \kappa^2}{\epsilon^2} \log \frac{C}{\epsilon} \right)$
features.
This is precisely the same dependence on $l$ and $\epsilon$ as for
random Fourier features.

A limitation of Theorem~\ref{thm:approx} is that it only shows approximation
of the limiting kernel rather than direct approximation of functions
in the RKHS. 
A more detailed analysis of the convergence to RKHS 
is contained in the work of \citet{bach2017a},
whereas
\citet{rudi2017} directly analyze the generalization ability of these approximations.
\citet{sun2018a}
show even faster rates which also apply to SVMs,
 assuming that the features are compatible
(``optimized'') for the learning problem.
Also, the techniques of \citet{sutherland2015} and \citet{sriperumbudur2015}
could be used to improve our constants and prove convergence in other $L^p$ norms.

In the sparse case,
we must extend our probability space
to capture the randomness of 
(1) the degrees, (2) the neighborhoods conditional on the degree, and
(3) the weight vectors conditional on the degree and neighborhood.
The degrees are distributed independently according to $d_i \sim D$,
with some abuse of notation since we also use $D(d)$ to represent the probability mass function.
We shall always think of the neighborhoods $\mathcal{N} \sim \nu|d$ as chosen uniformly among all
$d$ element subsets, where $\nu|d$ represents this conditional distribution.
Finally, given a neighborhood of some degree, 
the nonzero weights and bias are drawn from a distribution 
$(\vec w, b) \sim \mu| d$
on $\R^{d+1}$.
For simpler notation, we do not show any dependence on the neighborhood here,
since we will always take the actual weight values 
to not depend on the particular neighborhood $\mathcal{N}$.
However, strictly speaking, the weights do depend on $\mathcal{N}$
because that determines their support.
Finally, 
we use $\E$ to denote expectation over all variables (degree, neighborhood, and weights),
whereas we use $\E_{\mu|d}$ for the expectation under $\mu|d$ for a given degree.

\begin{corollary}[Kernel approximation with sparse features]
  Assume that $\vec{x} \in \mathcal{X} \subset \R^l$
  and that $\mathcal{X}$ is compact, $\Delta = \mathrm{diam}(\mathcal{X})$, and 
  the null vector $0 \in \mathcal{X}$.
  Let the degrees $d$ follow the degree distribution $D$ on $[l]$.
  For every $d \in [l]$, 
  let $\mu | d$ denote the conditional distributions for $(\vec w, b)$ 
  on $\R^{d +1}$
  and assume that these have finite second moments.
  Let
  $h(\cdot)$ be a
  nonlinearity which is $L$-Lipschitz continuous,
  and define the random feature 
  $\phi : \R^l \to \R$
  by 
  $\phi(\vec x) = h ( \vec w^\intercal \vec x - b)$,
  where $\vec w$ follows the degree distribution model.
  We assume that
  the following hold for all 
  $\vec x_\mathcal{N} \in \mathcal{X}_\mathcal{N}$
  with $|\mathcal{N}| = d$,
  and for all $1 \leq d \leq l$:
  $|\phi(\vec x_\mathcal{N})|^2 \leq \kappa$ almost surely under $\mu | d$,
  $\E \left[ |\phi(\vec x_\mathcal{N})|^2 | d \right] < \infty$, and
  $\E [ \phi(\vec x_\mathcal{N}) \phi(\vec x_\mathcal{N}') | d]
    = k_d^\mathrm{reg} (\vec x_\mathcal{N}, \vec x_\mathcal{N}')$.

  Then 
  $
  \sup_{\vec x, \vec x' \in \mathcal{X}}
  \left| 
    \frac{1}{m} \bm{\phi}(\vec x)^\intercal \bm{\phi} (\vec x') 
    -
    k_D^\mathrm{dist} (\vec x, \vec x')
  \right| 
  \leq 
  \epsilon,
  $
  with probability
  at least
  \[  
  1 
  - 
  2^8 \left( \frac{\kappa L \Delta \sqrt{\E \|\vec w \|^2 + 3 (\E \|\vec w \|)^2}}{\epsilon} \right)^2 
  \exp \left( \frac{-m \epsilon^2}{4 (2 l + 2) \kappa^2} \right).
  \]
  The kernels
  $k_d^\mathrm{reg} (\vec z, \vec z')$
  and
  $k_D^\mathrm{dist} (\vec x, \vec x')$ 
  are given by equations \eqref{eq:K_reg} and \eqref{eq:K_deg_dist}.
  \label{cor:approx}
\end{corollary}

\begin{proof}
  It suffices to show that conditions (1--3) on the 
  conditional distributions
  $\mu | d$, $d \in [l]$,
  imply conditions (1--3) in Theorem~\ref{thm:approx}.
  Conditions (1) and (2) clearly hold,
  since the distribution $D$ has finite support.
  By construction, 
  $ \E \, \phi(\vec x) \phi(\vec x')
  = \E [ \E [ \phi(\vec x_\mathcal{N}) \phi(\vec x'_\mathcal{N}) | d ] ]
  = \E [ k_d^\mathrm{reg}(\vec x_\mathcal{N}, \vec x'_\mathcal{N}) ]
  =  k_D^\mathrm{dist} (\vec x, \vec x')$, 
  which concludes the proof.
\end{proof}

\paragraph{Differences of sparsity}
The only difference we find with sparse random features is in the
terms $\E \| \vec w \|^2$ and $\E \| \vec w \|$, since
sparsity adds variance to the weights.
This suggests that scaling the weights so that $\E_{\mu | d} \| \vec w \|^2$
is constant for all $d$ is a good idea.
For example, setting
$(\vec w_i)_{\mathcal{N}_i} \sim N(0, \sigma^2 d_i^{-1} \mat{I}_{d_i})$,
the random variables 
$\| \vec w_i \|^2 \sim \sigma^2 d_i^{-1} \chi^2 (d_i)$
and 
$\| \vec w_i \| \sim \sigma d_i^{-1/2} \chi (d_i)$.
Then $\E \| \vec w_i \|^2 = \sigma^2$
irregardless of $d_i$ 
and 
$\E \| \vec w_i \| = \sigma (1 + o(d_i))$.
With this choice, the number of sparse features 
needed to achieve an error $\epsilon$  is the same as in the dense case,
up to a small constant factor.
This is perhaps remarkable since there could be as many as 
$\sum_{d=0}^l \binom{l}{d} = 2^l$
terms in the expression of $k_D^\mathrm{dist} (\vec x, \vec x')$.
However, the random feature expansion does not need to approximate all of these terms
well, just their average.

\begin{proof}[Proof of Theorem~\ref{thm:approx}]
  We follow the approach of Claim 1 in \citep{rahimi2008},
  a similar result for random Fourier features but 
  which crucially uses the fact that the trigonometric functions are differentiable and bounded.
  For simplicity of notation, let $\bm\xi = (\vec x, \vec x')$ 
  and define the {\em direct sum norm} on 
  $\mathcal{X}^+ = \mathcal{X} \oplus \mathcal{X}$ as $\| \bm\xi \|_+ = \| \vec x\| + \| \vec x' \|$.
  Under this norm $\mathcal{X}^+$ is a Banach space but not a Hilbert space,
  however this will not matter.
  For $i=1, \ldots, m$, let
  \begin{align*}
    f_i (\bm\xi) &= \phi_i(\vec x) \phi_i( \vec x'), \\
    g_i (\bm\xi) &= \phi_i(\vec x) \phi_i( \vec x') - k(\vec x, \vec x') \\
    &= f_i(\bm\xi) - \E f_i(\bm\xi),
  \end{align*}
  and note that these $g_i$ are i.i.d., centered random variables.
  By assumptions (1) and (2),
  $f_i$ and $g_i$ are absolutely integrable and 
  $k(\vec x, \vec x') = \E \, \phi_i(\vec x) \phi_i (\vec x')$.
  Denote their mean by
  \begin{align*}
 \bar{g}(\bm\xi) 
    = \frac{1}{m} \bm{\phi}(\vec x)^\intercal \bm{\phi} (\vec x') - k(\vec x, \vec x')
    = \frac{1}{m} \sum_{i=1}^m g_i(\bm\xi) .
  \end{align*}
  Our goal is to show that 
  $|\bar{g} (\bm\xi)| \leq \epsilon$ for all $\bm\xi \in \mathcal{X}^+$ 
  with sufficiently high probability.

  The space $\mathcal{X}^+$ is compact and $2l$-dimensional, 
  and it has diameter at most twice the diameter of $\mathcal{X}$ under the sum norm.
  Thus we can cover $\mathcal{X}^+$ 
  with an $\epsilon$-net using at most 
  $T = (4\Delta/ R )^{2l}$ balls of radius $R$.
  Call the centers of these balls $\bm\xi_i$ for $i = 1, \ldots, T$,
  and let $\bar{L}$ denote the Lipschitz constant of $\bar{g}$ with respect to the sum norm.
  Then we can show that $| \bar{g}(\bm\xi) | \leq \epsilon$ for all 
  $\bm\xi \in \mathcal{X}^+$ 
  if we show that
  \begin{enumerate}
  \item $\bar{L} \leq \frac{\epsilon}{2R}$, and
  \item $|\bar{g}(\bm\xi_i) | \leq \frac{\epsilon}{2}$ for all $i$.
  \end{enumerate}

  First, we bound the Lipschitz constant of $g_i$ with respect to the sum norm 
  $\| \cdot \|_+$.
  Since $h$ is $L$-Lipschitz, we have that $\phi_i$ is Lipschitz with constant 
  $L \|\vec w_i\|$.
  Thus, letting
  $\bm \xi' = \bm\xi + (\bm\delta,\bm\delta')$,
  \begin{align*}
    2|f_i(\bm\xi) - f_i(\bm\xi')| 
    &\leq 
      |\phi_i(\vec x+\bm\delta) \phi_i(\vec x' + \bm\delta')
      - 
      \phi_i(\vec x + \bm\delta) \phi_i(\vec x')|  \\
    & \quad +      |\phi_i(\vec x+\bm\delta) \phi_i(\vec x' + \bm\delta')
      - 
      \phi_i(\vec x) \phi_i(\vec x' + \bm\delta')| \\
    & \quad +      |\phi_i(\vec x+\bm\delta) \phi_i(\vec x')
      - 
      \phi_i(\vec x) \phi_i(\vec x')| \\
    & \quad +      |\phi_i(\vec x) \phi_i(\vec x' + \bm\delta')
      - 
      \phi_i(\vec x) \phi_i(\vec x')| \\
    & \leq 2 L \| \vec w_i \| \cdot
    \sup_{\vec x \in \mathcal{X}}  | \phi_i(\vec x)|
      \cdot  ( \| \bm\delta \| + \| \bm\delta' \| ) 
    \\
    & = 2 \kappa L \| \vec w_i \| \cdot \| \bm \xi - \bm\xi' \|_+,
  \end{align*}
  we have that $f_i$ has Lipschitz constant $\kappa L \|\vec w_i \|$.
  This implies that $g_i$ has Lipschitz constant $\leq \kappa L (\| \vec w_i \| + \E \|\vec w\|)$.
  
  Let $\bar{L}$ denote the Lipschitz constant of $\bar{g}$.
  Note that $\E \bar{L} \leq 2 \kappa L \E \| \vec w \|$.
  Also, 
  \begin{align*}
    \E \bar{L}^2 
    &\leq L^2 \kappa^2 \E \left( \| \vec w \| + \E \| \vec w \| \right )^2 \\
    & = L^2 \kappa^2 \left( \E \| \vec w \|^2 + 3 ( \E \| \vec w \| )^2 \right).
  \end{align*}
  Markov's inequality states that
  $\P [ \bar{L}^2 > t^2 ] \leq \E [ \bar{L}^2 ]/ t^2$.
  Letting $t = \frac{\epsilon}{2 R}$, we find that
  \begin{equation}
    \label{eq:Lipschitz_bnd}
    \P [ \bar{L} > t  ] 
    = 
    \P \left[ \bar{L} > \frac{\epsilon}{2R}  \right] 
    \leq 
    L^2 \kappa^2 \left( \E \| \vec w \|^2 + 3 ( \E \| \vec w \| )^2 \right) 
    \left( \frac{2R}{\epsilon} \right)^2.
  \end{equation}

  Now we would like to show that $|\bar{g}(\bm\xi_i)| \leq \epsilon /2$ for all
  $i = 1, \ldots, T$
  anchors in the $\epsilon$-net.
  A straightforward application of Hoeffding's inequality and a union bound shows that 
  \begin{equation}
    \label{eq:Hoeffding_bnd}
    \P \left[ |\bar{g}(\bm\xi_i) | > \frac{\epsilon }{2} \mbox{ for all $i$}  \right] 
    \leq 
    2 T \exp \left( \frac{- m \epsilon^2}{8 \kappa^4} \right) ,
  \end{equation}
  since $|f_i(\bm\xi)| \leq \kappa^2$.

  Combining equations \eqref{eq:Lipschitz_bnd} and \eqref{eq:Hoeffding_bnd} results in 
  a probability of failure
  \begin{align}
  \P \left[ \sup_{\bm\xi \in \mathcal{X}^+} |\bar{g}(\bm\xi)| \geq \epsilon  \right]
    & \leq 
      2 \left( \frac{4\Delta}{R} \right)^{2l} 
      \exp \left( \frac{- m \epsilon^2}{8 \kappa^2} \right)
      + 
      L^2 \kappa^2 ( \E \| \vec w \|^2 + 3 ( \E \| \vec w \|)^2 ) \left( \frac{2R}{\epsilon} \right)^2 \nonumber \\
    &= a R^{-2l} + b R^2 . \label{eq:pr_failure}
  \end{align}
  Set $R = (a/b)^{\frac{1}{2l+2}}$, so that the probability \eqref{eq:pr_failure} has the form,
  $1 - 2 a^{\frac{2}{2l+2}} b^{\frac{2l}{2l+2}}$.
  Thus the probability of failure satisfies
  \begin{align*}
    \Pr \left[ \sup_{\bm\xi \in \mathcal{X}^+} |\bar{g}(\bm\xi)| \geq \epsilon  \right] 
    &\leq 
    2 a^{\frac{2}{2l+2}} b^{\frac{2l}{2l+2}}\\
    &= 2 \cdot 2^{\frac{2}{2l+2}} 
      \left( \frac{8 \kappa L \Delta
      \sqrt{ \E \| \vec w \|^2 + 3 ( \E \| \vec w \|)^2 }}{\epsilon} \right)^{\frac{4l}{2l+2}}
      \exp \left( \frac{-m \epsilon^2}{4 (2 l + 2) \kappa^2} \right) \\
    &\leq 2^8 \left(
      \frac{ \kappa L \Delta \sqrt{\E \| \vec w \|^2 + 3 ( \E \| \vec w \|)^2}}{\epsilon} 
      \right)^2
      \exp \left( \frac{-m \epsilon^2}{4 (2 l + 2) \kappa^2} \right),
  \end{align*}
  for all $l \in \mathbb{N}$, 
  assuming $\Delta \kappa L \sqrt{\E \| \vec w \|^2 + 3 ( \E \| \vec w \|)^2} > \epsilon$.
  Considering the complementary event concludes the proof.
\end{proof}

{
\setlength{\bibsep}{0pt plus 0.3ex}
\small
\bibliographystyle{unsrtnat}
\bibliography{Library}

\begin{thebibliography}{11}
\providecommand{\natexlab}[1]{#1}
\providecommand{\url}[1]{\texttt{#1}}
\expandafter\ifx\csname urlstyle\endcsname\relax
  \providecommand{\doi}[1]{doi: #1}\else
  \providecommand{\doi}{doi: \begingroup \urlstyle{rm}\Url}\fi

\bibitem[Kandasamy and Yu(2016)]{kandasamy2016}
Kirthevasan Kandasamy and Yaoliang Yu.
\newblock Additive {{Approximations}} in {{High Dimensional Nonparametric
  Regression}} via the {{SALSA}}.
\newblock In \emph{International {{Conference}} on {{Machine Learning}}}, pages
  69--78, June 2016.

\bibitem[Rahimi and Recht(2008)]{rahimi2008}
Ali Rahimi and Benjamin Recht.
\newblock Random {{Features}} for {{Large}}-{{Scale Kernel Machines}}.
\newblock In J.~C. Platt, D.~Koller, Y.~Singer, and S.~T. Roweis, editors,
  \emph{Advances in {{Neural Information Processing Systems}} 20}, pages
  1177--1184. {Curran Associates, Inc.}, 2008.

\bibitem[Smola et~al.(2001)Smola, {\'O}v{\'a}ri, and Williamson]{smola2001}
Alex~J. Smola, Zolt{\'a}n~L. {\'O}v{\'a}ri, and Robert~C Williamson.
\newblock Regularization with {{Dot}}-{{Product Kernels}}.
\newblock In T.~K. Leen, T.~G. Dietterich, and V.~Tresp, editors,
  \emph{Advances in {{Neural Information Processing Systems}} 13}, pages
  308--314. {MIT Press}, 2001.

\bibitem[Bach(2017{\natexlab{a}})]{bach2017}
Francis Bach.
\newblock Breaking the {{Curse}} of {{Dimensionality}} with {{Convex Neural
  Networks}}.
\newblock \emph{Journal of Machine Learning Research}, 18\penalty0
  (19):\penalty0 1--53, 2017{\natexlab{a}}.

\bibitem[Cho and Saul(2009)]{cho2009}
Youngmin Cho and Lawrence~K. Saul.
\newblock Kernel {{Methods}} for {{Deep Learning}}.
\newblock In Y.~Bengio, D.~Schuurmans, J.~D. Lafferty, C.~K.~I. Williams, and
  A.~Culotta, editors, \emph{Advances in {{Neural Information Processing
  Systems}} 22}, pages 342--350. {Curran Associates, Inc.}, 2009.

\bibitem[Cho and Saul(2011)]{cho2011}
Youngmin Cho and Lawrence~K. Saul.
\newblock Analysis and {{Extension}} of {{Arc}}-{{Cosine Kernels}} for {{Large
  Margin Classification}}.
\newblock \emph{arXiv:1112.3712 [cs]}, December 2011.

\bibitem[Bach(2017{\natexlab{b}})]{bach2017a}
Francis Bach.
\newblock On the {{Equivalence Between Kernel Quadrature Rules}} and {{Random
  Feature Expansions}}.
\newblock \emph{J. Mach. Learn. Res.}, 18\penalty0 (1):\penalty0 714--751,
  January 2017{\natexlab{b}}.
\newblock ISSN 1532-4435.

\bibitem[Sutherland and Schneider(2015)]{sutherland2015}
Dougal~J. Sutherland and Jeff Schneider.
\newblock On the {{Error}} of {{Random Fourier Features}}.
\newblock \emph{arXiv:1506.02785 [cs, stat]}, June 2015.

\bibitem[Sriperumbudur and Szabo(2015)]{sriperumbudur2015}
Bharath Sriperumbudur and Zoltan Szabo.
\newblock Optimal {{Rates}} for {{Random Fourier Features}}.
\newblock In C.~Cortes, N.~D. Lawrence, D.~D. Lee, M.~Sugiyama, and R.~Garnett,
  editors, \emph{Advances in {{Neural Information Processing Systems}} 28},
  pages 1144--1152. {Curran Associates, Inc.}, 2015.

\bibitem[Rudi and Rosasco(2017)]{rudi2017}
Alessandro Rudi and Lorenzo Rosasco.
\newblock Generalization {{Properties}} of {{Learning}} with {{Random
  Features}}.
\newblock In I.~Guyon, U.~V. Luxburg, S.~Bengio, H.~Wallach, R.~Fergus,
  S.~Vishwanathan, and R.~Garnett, editors, \emph{Advances in {{Neural
  Information Processing Systems}} 30}, pages 3215--3225. {Curran Associates,
  Inc.}, 2017.

\bibitem[Sun et~al.(2018)Sun, Gilbert, and Tewari]{sun2018a}
Yitong Sun, Anna Gilbert, and Ambuj Tewari.
\newblock But {{How Does It Work}} in {{Theory}}? {{Linear SVM}} with {{Random
  Features}}.
\newblock \emph{arXiv:1809.04481 [cs, stat]}, September 2018.

\end{thebibliography}


\begin{thebibliography}{59}
\providecommand{\natexlab}[1]{#1}
\providecommand{\url}[1]{\texttt{#1}}
\expandafter\ifx\csname urlstyle\endcsname\relax
  \providecommand{\doi}[1]{doi: #1}\else
  \providecommand{\doi}{doi: \begingroup \urlstyle{rm}\Url}\fi

\bibitem[Bach(2017{\natexlab{a}})]{bach2017}
Francis Bach.
\newblock Breaking the {{Curse}} of {{Dimensionality}} with {{Convex Neural
  Networks}}.
\newblock \emph{Journal of Machine Learning Research}, 18\penalty0
  (19):\penalty0 1--53, 2017{\natexlab{a}}.

\bibitem[Jacot et~al.(2018)Jacot, Gabriel, and Hongler]{jacot2018}
Arthur Jacot, Franck Gabriel, and Cl{\'e}ment Hongler.
\newblock Neural {{Tangent Kernel}}: {{Convergence}} and {{Generalization}} in
  {{Neural Networks}}.
\newblock \emph{arXiv:1806.07572 [cs, math, stat]}, June 2018.

\bibitem[Chizat and Bach(2018)]{chizat2018}
Lenaic Chizat and Francis Bach.
\newblock On the {{Global Convergence}} of {{Gradient Descent}} for
  {{Over}}-parameterized {{Models}} using {{Optimal Transport}}.
\newblock \emph{arXiv:1805.09545 [cs, math, stat]}, May 2018.

\bibitem[Mei et~al.(2018)Mei, Montanari, and Nguyen]{mei2018}
Song Mei, Andrea Montanari, and Phan-Minh Nguyen.
\newblock A {{Mean Field View}} of the {{Landscape}} of {{Two}}-{{Layers Neural
  Networks}}.
\newblock \emph{arXiv:1804.06561 [cond-mat, stat]}, April 2018.

\bibitem[Rotskoff and {Vanden-Eijnden}(2018)]{rotskoff2018}
Grant~M. Rotskoff and Eric {Vanden-Eijnden}.
\newblock Trainability and {{Accuracy}} of {{Neural Networks}}: {{An
  Interacting Particle System Approach}}.
\newblock \emph{arXiv:1805.00915 [cond-mat, stat]}, May 2018.

\bibitem[Venturi et~al.(2018)Venturi, Bandeira, and Bruna]{venturi2018}
Luca Venturi, Afonso~S. Bandeira, and Joan Bruna.
\newblock Spurious {{Valleys}} in {{Two}}-layer {{Neural Network Optimization
  Landscapes}}.
\newblock \emph{arXiv:1802.06384 [cs, math, stat]}, February 2018.

\bibitem[Rosenblatt(1958)]{rosenblatt1958}
F.~Rosenblatt.
\newblock The {{Perceptron}}: {{A Probabilistic Model}} for {{Information
  Storage}} and {{Organization}} in the {{Brain}}.
\newblock \emph{Psychological Review}, 65\penalty0 (6):\penalty0 386--408,
  1958.

\bibitem[Broomhead and Lowe(1988)]{broomhead1988}
D.~S. Broomhead and David Lowe.
\newblock Radial {{Basis Functions}}, {{Multi}}-{{Variable Functional
  Interpolation}} and {{Adaptive Networks}}.
\newblock Technical Report RSRE-MEMO-4148, {Royal Signals and Radar
  Establishment Malvern (UK)}, March 1988.

\bibitem[Igelnik and Pao(1995)]{igelnik1995}
B.~Igelnik and Yoh-Han Pao.
\newblock Stochastic choice of basis functions in adaptive function
  approximation and the functional-link net.
\newblock \emph{IEEE Transactions on Neural Networks}, 6\penalty0 (6):\penalty0
  1320--1329, November 1995.
\newblock ISSN 1045-9227.
\newblock \doi{10.1109/72.471375}.

\bibitem[Neal(1996)]{neal1996}
Radford~M. Neal.
\newblock Priors for {{Infinite Networks}}.
\newblock In \emph{Bayesian {{Learning}} for {{Neural Networks}}}, Lecture
  {{Notes}} in {{Statistics}}, pages 29--53. {Springer, New York, NY}, 1996.
\newblock ISBN 978-0-387-94724-2 978-1-4612-0745-0.
\newblock \doi{10.1007/978-1-4612-0745-0_2}.

\bibitem[Williams(1997)]{williams1997}
Christopher K.~I. Williams.
\newblock Computing with {{Infinite Networks}}.
\newblock In M.~C. Mozer, M.~I. Jordan, and T.~Petsche, editors, \emph{Advances
  in {{Neural Information Processing Systems}} 9}, pages 295--301. {MIT Press},
  1997.

\bibitem[Wang and Wan(2008)]{wang2008}
L.~P. Wang and C.~R. Wan.
\newblock Comments on "{{The Extreme Learning Machine}}".
\newblock \emph{IEEE Transactions on Neural Networks}, 19\penalty0
  (8):\penalty0 1494--1495, August 2008.
\newblock ISSN 1045-9227.
\newblock \doi{10.1109/TNN.2008.2002273}.

\bibitem[Scardapane and Wang(2017)]{scardapane2017}
Simone Scardapane and Dianhui Wang.
\newblock Randomness in neural networks: An overview.
\newblock \emph{Wiley Interdisciplinary Reviews: Data Mining and Knowledge
  Discovery}, 7\penalty0 (2):\penalty0 e1200, 2017.
\newblock ISSN 1942-4795.
\newblock \doi{10.1002/widm.1200}.

\bibitem[Rahimi and Recht(2008{\natexlab{a}})]{rahimi2008}
Ali Rahimi and Benjamin Recht.
\newblock Random {{Features}} for {{Large}}-{{Scale Kernel Machines}}.
\newblock In J.~C. Platt, D.~Koller, Y.~Singer, and S.~T. Roweis, editors,
  \emph{Advances in {{Neural Information Processing Systems}} 20}, pages
  1177--1184. {Curran Associates, Inc.}, 2008{\natexlab{a}}.

\bibitem[Rahimi and Recht(2008{\natexlab{b}})]{rahimi2008a}
A.~Rahimi and B.~Recht.
\newblock Uniform approximation of functions with random bases.
\newblock In \emph{2008 46th {{Annual Allerton Conference}} on
  {{Communication}}, {{Control}}, and {{Computing}}}, pages 555--561, September
  2008{\natexlab{b}}.
\newblock \doi{10.1109/ALLERTON.2008.4797607}.

\bibitem[Ganguli and Sompolinsky(2012)]{ganguli2012}
Surya Ganguli and Haim Sompolinsky.
\newblock Compressed {{Sensing}}, {{Sparsity}}, and {{Dimensionality}} in
  {{Neuronal Information Processing}} and {{Data Analysis}}.
\newblock \emph{Annual Review of Neuroscience}, 35\penalty0 (1):\penalty0
  485--508, 2012.
\newblock \doi{10.1146/annurev-neuro-062111-150410}.

\bibitem[Caron et~al.(2013)Caron, Ruta, Abbott, and Axel]{caron2013}
Sophie J.~C. Caron, Vanessa Ruta, L.~F. Abbott, and Richard Axel.
\newblock Random convergence of olfactory inputs in the {{Drosophila}} mushroom
  body.
\newblock \emph{Nature}, 497\penalty0 (7447):\penalty0 113--117, May 2013.
\newblock ISSN 0028-0836.
\newblock \doi{10.1038/nature12063}.

\bibitem[Caron(2013)]{caron2013a}
Sophie J.~C. Caron.
\newblock Brains {{Don}}'t {{Play Dice}}\textemdash{}or {{Do They}}?
\newblock \emph{Science}, 342\penalty0 (6158):\penalty0 574--574, November
  2013.
\newblock ISSN 0036-8075, 1095-9203.
\newblock \doi{10.1126/science.1245982}.

\bibitem[Harris et~al.(2017)Harris, Dashevskiy, Mendoza, Garcia, Ramirez, and
  {Shea-Brown}]{harris2017}
Kameron~Decker Harris, Tatiana Dashevskiy, Joshua Mendoza, Alfredo~J. Garcia,
  Jan-Marino Ramirez, and Eric {Shea-Brown}.
\newblock Different roles for inhibition in the rhythm-generating respiratory
  network.
\newblock \emph{Journal of Neurophysiology}, 118\penalty0 (4):\penalty0
  2070--2088, October 2017.
\newblock ISSN 0022-3077, 1522-1598.
\newblock \doi{10.1152/jn.00174.2017}.

\bibitem[{Litwin-Kumar} et~al.(2017){Litwin-Kumar}, Harris, Axel, Sompolinsky,
  and Abbott]{litwin-kumar2017}
Ashok {Litwin-Kumar}, Kameron~Decker Harris, Richard Axel, Haim Sompolinsky,
  and L.~F. Abbott.
\newblock Optimal {{Degrees}} of {{Synaptic Connectivity}}.
\newblock \emph{Neuron}, 93\penalty0 (5):\penalty0 1153--1164.e7, March 2017.
\newblock ISSN 0896-6273.
\newblock \doi{10.1016/j.neuron.2017.01.030}.

\bibitem[{Cayco-Gajic} and Silver(2019)]{cayco-gajic2019}
N.~Alex {Cayco-Gajic} and R.~Angus Silver.
\newblock Re-evaluating {{Circuit Mechanisms Underlying Pattern Separation}}.
\newblock \emph{Neuron}, 101\penalty0 (4):\penalty0 584--602, February 2019.
\newblock ISSN 08966273.
\newblock \doi{10.1016/j.neuron.2019.01.044}.

\bibitem[Wolff and Strausfeld(2016)]{wolff2016}
Gabriella~H. Wolff and Nicholas~J. Strausfeld.
\newblock Genealogical correspondence of a forebrain centre implies an
  executive brain in the protostome\textendash{}deuterostome bilaterian
  ancestor.
\newblock \emph{Philosophical Transactions of the Royal Society B: Biological
  Sciences}, 371\penalty0 (1685):\penalty0 20150055, January 2016.
\newblock \doi{10.1098/rstb.2015.0055}.

\bibitem[Han et~al.(2015{\natexlab{a}})Han, Pool, Tran, and Dally]{han2015}
Song Han, Jeff Pool, John Tran, and William Dally.
\newblock Learning both {{Weights}} and {{Connections}} for {{Efficient Neural
  Network}}.
\newblock In C.~Cortes, N.~D. Lawrence, D.~D. Lee, M.~Sugiyama, and R.~Garnett,
  editors, \emph{Advances in {{Neural Information Processing Systems}} 28},
  pages 1135--1143. {Curran Associates, Inc.}, 2015{\natexlab{a}}.

\bibitem[Han et~al.(2015{\natexlab{b}})Han, Mao, and Dally]{han2015a}
Song Han, Huizi Mao, and William~J. Dally.
\newblock Deep {{Compression}}: {{Compressing Deep Neural Networks}} with
  {{Pruning}}, {{Trained Quantization}} and {{Huffman Coding}}.
\newblock \emph{arXiv:1510.00149 [cs]}, October 2015{\natexlab{b}}.

\bibitem[Wen et~al.(2016)Wen, Wu, Wang, Chen, and Li]{wen2016}
Wei Wen, Chunpeng Wu, Yandan Wang, Yiran Chen, and Hai Li.
\newblock Learning {{Structured Sparsity}} in {{Deep Neural Networks}}.
\newblock In D.~D. Lee, M.~Sugiyama, U.~V. Luxburg, I.~Guyon, and R.~Garnett,
  editors, \emph{Advances in {{Neural Information Processing Systems}} 29},
  pages 2074--2082. {Curran Associates, Inc.}, 2016.

\bibitem[Mairal et~al.(2014)Mairal, Koniusz, Harchaoui, and Schmid]{mairal2014}
Julien Mairal, Piotr Koniusz, Zaid Harchaoui, and Cordelia Schmid.
\newblock Convolutional {{Kernel Networks}}.
\newblock \emph{arXiv:1406.3332 [cs, stat]}, June 2014.

\bibitem[Jones et~al.(2019)Jones, Roulet, and Harchaoui]{jones2019}
Corinne Jones, Vincent Roulet, and Zaid Harchaoui.
\newblock Kernel-based {{Translations}} of {{Convolutional Networks}}.
\newblock \emph{arXiv:1903.08131 [cs, math, stat]}, March 2019.

\bibitem[Wahba(1990)]{wahba1990}
Grace Wahba.
\newblock \emph{Spline {{Models}} for {{Observational Data}}}.
\newblock {SIAM}, September 1990.
\newblock ISBN 978-0-89871-244-5.

\bibitem[Bach(2009)]{bach2009}
Francis~R. Bach.
\newblock Exploring {{Large Feature Spaces}} with {{Hierarchical Multiple
  Kernel Learning}}.
\newblock In D.~Koller, D.~Schuurmans, Y.~Bengio, and L.~Bottou, editors,
  \emph{Advances in {{Neural Information Processing Systems}} 21}, pages
  105--112. {Curran Associates, Inc.}, 2009.

\bibitem[Duvenaud et~al.(2011)Duvenaud, Nickisch, and Rasmussen]{duvenaud2011}
David~K Duvenaud, Hannes Nickisch, and Carl~E. Rasmussen.
\newblock Additive {{Gaussian Processes}}.
\newblock In J.~{Shawe-Taylor}, R.~S. Zemel, P.~L. Bartlett, F.~Pereira, and
  K.~Q. Weinberger, editors, \emph{Advances in {{Neural Information Processing
  Systems}} 24}, pages 226--234. {Curran Associates, Inc.}, 2011.

\bibitem[Kandasamy and Yu(2016)]{kandasamy2016}
Kirthevasan Kandasamy and Yaoliang Yu.
\newblock Additive {{Approximations}} in {{High Dimensional Nonparametric
  Regression}} via the {{SALSA}}.
\newblock In \emph{International {{Conference}} on {{Machine Learning}}}, pages
  69--78, June 2016.

\bibitem[Hastie et~al.(2009)Hastie, Tibshirani, and Friedman]{hastie2009}
Trevor Hastie, Robert Tibshirani, and Jerome Friedman.
\newblock \emph{The {{Elements}} of {{Statistical Learning}}: {{Data Mining}},
  {{Inference}}, and {{Prediction}}}.
\newblock {Springer-Verlag New York}, {New York, NY}, 2009.
\newblock ISBN 978-0-387-84858-7.
\newblock OCLC: 428882834.

\bibitem[Stone(1985)]{stone1985}
Charles~J. Stone.
\newblock Additive {{Regression}} and {{Other Nonparametric Models}}.
\newblock \emph{The Annals of Statistics}, 13\penalty0 (2):\penalty0 689--705,
  June 1985.
\newblock ISSN 0090-5364, 2168-8966.
\newblock \doi{10.1214/aos/1176349548}.

\bibitem[Stone(1986)]{stone1986}
Charles~J. Stone.
\newblock The {{Dimensionality Reduction Principle}} for {{Generalized Additive
  Models}}.
\newblock \emph{The Annals of Statistics}, 14\penalty0 (2):\penalty0 590--606,
  June 1986.
\newblock ISSN 0090-5364, 2168-8966.
\newblock \doi{10.1214/aos/1176349940}.

\bibitem[Hinton et~al.(2012)Hinton, Srivastava, Krizhevsky, Sutskever, and
  Salakhutdinov]{hinton2012}
Geoffrey~E. Hinton, Nitish Srivastava, Alex Krizhevsky, Ilya Sutskever, and
  Ruslan~R. Salakhutdinov.
\newblock Improving neural networks by preventing co-adaptation of feature
  detectors.
\newblock \emph{arXiv:1207.0580 [cs]}, July 2012.

\bibitem[Srivastava(2013)]{srivastava2013}
Nitish Srivastava.
\newblock \emph{Improving {{Neural Networks}} with {{Dropout}}}.
\newblock {University of Toronto}, 2013.

\bibitem[Duvenaud et~al.(2014)Duvenaud, Rippel, Adams, and
  Ghahramani]{duvenaud2014}
David Duvenaud, Oren Rippel, Ryan~P. Adams, and Zoubin Ghahramani.
\newblock Avoiding pathologies in very deep networks.
\newblock \emph{arXiv:1402.5836 [cs, stat]}, February 2014.

\bibitem[Lecuyer et~al.(2018)Lecuyer, Atlidakis, Geambasu, Hsu, and
  Jana]{lecuyer2018}
Mathias Lecuyer, Vaggelis Atlidakis, Roxana Geambasu, Daniel Hsu, and Suman
  Jana.
\newblock Certified {{Robustness}} to {{Adversarial Examples}} with
  {{Differential Privacy}}.
\newblock \emph{arXiv:1802.03471 [cs, stat]}, February 2018.

\bibitem[Cohen et~al.(2019)Cohen, Rosenfeld, and Kolter]{cohen2019}
Jeremy~M. Cohen, Elan Rosenfeld, and J.~Zico Kolter.
\newblock Certified {{Adversarial Robustness}} via {{Randomized Smoothing}}.
\newblock \emph{arXiv:1902.02918 [cs, stat]}, February 2019.

\bibitem[Salman et~al.(2019)Salman, Yang, Li, Zhang, Zhang, Razenshteyn, and
  Bubeck]{salman2019}
Hadi Salman, Greg Yang, Jerry Li, Pengchuan Zhang, Huan Zhang, Ilya
  Razenshteyn, and Sebastien Bubeck.
\newblock Provably {{Robust Deep Learning}} via {{Adversarially Trained
  Smoothed Classifiers}}.
\newblock \emph{arXiv:1906.04584 [cs, stat]}, June 2019.

\bibitem[Huerta and Nowotny(2009)]{huerta2009}
Ram{\'o}n Huerta and Thomas Nowotny.
\newblock Fast and {{Robust Learning}} by {{Reinforcement Signals}}:
  {{Explorations}} in the {{Insect Brain}}.
\newblock \emph{Neural Computation}, 21\penalty0 (8):\penalty0 2123--2151,
  August 2009.
\newblock ISSN 0899-7667, 1530-888X.
\newblock \doi{10.1162/neco.2009.03-08-733}.

\bibitem[Delahunt and Kutz(2019)]{delahunt2019}
Charles~B. Delahunt and J.~Nathan Kutz.
\newblock Putting a bug in {{ML}}: {{The}} moth olfactory network learns to
  read {{MNIST}}.
\newblock \emph{Neural Networks}, 118:\penalty0 54--64, October 2019.
\newblock ISSN 0893-6080.
\newblock \doi{10.1016/j.neunet.2019.05.012}.

\bibitem[Delahunt and Kutz(2018)]{delahunt2018a}
Charles~B. Delahunt and J.~Nathan Kutz.
\newblock Insect cyborgs: {{Bio}}-mimetic feature generators improve machine
  learning accuracy on limited data.
\newblock \emph{arXiv:1808.08124 [cs, stat]}, August 2018.

\bibitem[Rigotti et~al.(2013)Rigotti, Barak, Warden, Wang, Daw, Miller, and
  Fusi]{rigotti2013}
Mattia Rigotti, Omri Barak, Melissa~R. Warden, Xiao-Jing Wang, Nathaniel~D.
  Daw, Earl~K. Miller, and Stefano Fusi.
\newblock The importance of mixed selectivity in complex cognitive tasks.
\newblock \emph{Nature}, 497\penalty0 (7451):\penalty0 585--590, May 2013.
\newblock ISSN 0028-0836.
\newblock \doi{10.1038/nature12160}.

\bibitem[Babadi and Sompolinsky(2014)]{babadi2014}
Baktash Babadi and Haim Sompolinsky.
\newblock Sparseness and {{Expansion}} in {{Sensory Representations}}.
\newblock \emph{Neuron}, 83\penalty0 (5):\penalty0 1213--1226, September 2014.
\newblock ISSN 0896-6273.
\newblock \doi{10.1016/j.neuron.2014.07.035}.

\bibitem[Meister(2015)]{meister2015}
Markus Meister.
\newblock On the dimensionality of odor space.
\newblock \emph{eLife}, 4:\penalty0 e07865, July 2015.
\newblock ISSN 2050-084X.
\newblock \doi{10.7554/eLife.07865}.

\bibitem[Mazzucato et~al.(2016)Mazzucato, Fontanini, and
  La~Camera]{mazzucato2016}
Luca Mazzucato, Alfredo Fontanini, and Giancarlo La~Camera.
\newblock Stimuli {{Reduce}} the {{Dimensionality}} of {{Cortical Activity}}.
\newblock \emph{Frontiers in Systems Neuroscience}, 10, 2016.
\newblock ISSN 1662-5137.
\newblock \doi{10.3389/fnsys.2016.00011}.

\bibitem[Gao et~al.(2017)Gao, Trautmann, Yu, Santhanam, Ryu, Shenoy, and
  Ganguli]{gao2017}
Peiran Gao, Eric Trautmann, Byron~M. Yu, Gopal Santhanam, Stephen Ryu, Krishna
  Shenoy, and Surya Ganguli.
\newblock A theory of multineuronal dimensionality, dynamics and measurement.
\newblock November 2017.
\newblock \doi{10.1101/214262}.

\bibitem[Mastrogiuseppe and Ostojic(2018)]{mastrogiuseppe2018}
Francesca Mastrogiuseppe and Srdjan Ostojic.
\newblock Linking {{Connectivity}}, {{Dynamics}}, and {{Computations}} in
  {{Low}}-{{Rank Recurrent Neural Networks}}.
\newblock \emph{Neuron}, 99\penalty0 (3):\penalty0 609--623.e29, August 2018.
\newblock ISSN 0896-6273.
\newblock \doi{10.1016/j.neuron.2018.07.003}.

\bibitem[Farrell et~al.(2019)Farrell, Recanatesi, Lajoie, and
  {Shea-Brown}]{farrell2019}
Matthew~S. Farrell, Stefano Recanatesi, Guillaume Lajoie, and Eric
  {Shea-Brown}.
\newblock Dynamic compression and expansion in a classifying recurrent network.
\newblock \emph{bioRxiv}, page 564476, March 2019.
\newblock \doi{10.1101/564476}.

\bibitem[Zhang(2005)]{zhang2005}
Tong Zhang.
\newblock Learning {{Bounds}} for {{Kernel Regression Using Effective Data
  Dimensionality}}.
\newblock \emph{Neural Computation}, 17\penalty0 (9):\penalty0 2077--2098,
  September 2005.
\newblock ISSN 0899-7667.
\newblock \doi{10.1162/0899766054323008}.

\bibitem[Smola et~al.(2001)Smola, {\'O}v{\'a}ri, and Williamson]{smola2001}
Alex~J. Smola, Zolt{\'a}n~L. {\'O}v{\'a}ri, and Robert~C Williamson.
\newblock Regularization with {{Dot}}-{{Product Kernels}}.
\newblock In T.~K. Leen, T.~G. Dietterich, and V.~Tresp, editors,
  \emph{Advances in {{Neural Information Processing Systems}} 13}, pages
  308--314. {MIT Press}, 2001.

\bibitem[Cho and Saul(2009)]{cho2009}
Youngmin Cho and Lawrence~K. Saul.
\newblock Kernel {{Methods}} for {{Deep Learning}}.
\newblock In Y.~Bengio, D.~Schuurmans, J.~D. Lafferty, C.~K.~I. Williams, and
  A.~Culotta, editors, \emph{Advances in {{Neural Information Processing
  Systems}} 22}, pages 342--350. {Curran Associates, Inc.}, 2009.

\bibitem[Cho and Saul(2011)]{cho2011}
Youngmin Cho and Lawrence~K. Saul.
\newblock Analysis and {{Extension}} of {{Arc}}-{{Cosine Kernels}} for {{Large
  Margin Classification}}.
\newblock \emph{arXiv:1112.3712 [cs]}, December 2011.

\bibitem[Bach(2017{\natexlab{b}})]{bach2017a}
Francis Bach.
\newblock On the {{Equivalence Between Kernel Quadrature Rules}} and {{Random
  Feature Expansions}}.
\newblock \emph{J. Mach. Learn. Res.}, 18\penalty0 (1):\penalty0 714--751,
  January 2017{\natexlab{b}}.
\newblock ISSN 1532-4435.

\bibitem[Sutherland and Schneider(2015)]{sutherland2015}
Dougal~J. Sutherland and Jeff Schneider.
\newblock On the {{Error}} of {{Random Fourier Features}}.
\newblock \emph{arXiv:1506.02785 [cs, stat]}, June 2015.

\bibitem[Sriperumbudur and Szabo(2015)]{sriperumbudur2015}
Bharath Sriperumbudur and Zoltan Szabo.
\newblock Optimal {{Rates}} for {{Random Fourier Features}}.
\newblock In C.~Cortes, N.~D. Lawrence, D.~D. Lee, M.~Sugiyama, and R.~Garnett,
  editors, \emph{Advances in {{Neural Information Processing Systems}} 28},
  pages 1144--1152. {Curran Associates, Inc.}, 2015.

\bibitem[Rudi and Rosasco(2017)]{rudi2017}
Alessandro Rudi and Lorenzo Rosasco.
\newblock Generalization {{Properties}} of {{Learning}} with {{Random
  Features}}.
\newblock In I.~Guyon, U.~V. Luxburg, S.~Bengio, H.~Wallach, R.~Fergus,
  S.~Vishwanathan, and R.~Garnett, editors, \emph{Advances in {{Neural
  Information Processing Systems}} 30}, pages 3215--3225. {Curran Associates,
  Inc.}, 2017.

\bibitem[Sun et~al.(2018)Sun, Gilbert, and Tewari]{sun2018a}
Yitong Sun, Anna Gilbert, and Ambuj Tewari.
\newblock But {{How Does It Work}} in {{Theory}}? {{Linear SVM}} with {{Random
  Features}}.
\newblock \emph{arXiv:1809.04481 [cs, stat]}, September 2018.

\end{thebibliography}
}

\end{document}